%% file: arxiv.tex
\documentclass[twoside,11pt]{article}

\usepackage{blindtext}

\usepackage[utf8]{inputenc}
\usepackage[T1]{fontenc}
\usepackage{lmodern}
\usepackage{hyperref}
\usepackage{url}            
\usepackage{booktabs}       
\usepackage{amsfonts}       
\usepackage[numbers]{natbib}
\usepackage{nicefrac}       
\usepackage{algorithm,algpseudocode}
\usepackage{amsfonts}
\usepackage{multicol}
\usepackage{amsmath}
\usepackage{amsthm}
\usepackage{amssymb}
\usepackage[title]{appendix}
\usepackage{bm}
\usepackage{courier}
\usepackage[usenames,dvipsnames]{color}
\usepackage{enumitem}
\usepackage{graphicx}
\usepackage{float}
\usepackage[margin=0.9in]{geometry}
\usepackage{url}
\usepackage{multirow}
\usepackage[dvipsnames]{xcolor}
\usepackage{epsfig}
\usepackage{graphicx}
\usepackage{subcaption}
\usepackage{wrapfig}   
\usepackage{caption}   

\usepackage[most]{tcolorbox}
\usepackage{xcolor}

\definecolor{boxgray}{RGB}{242,242,242}
\definecolor{boxgold}{RGB}{252,249,240}
\hypersetup{
	colorlinks,
	linkcolor={red!80!black},
	citecolor={blue!50!black},
	urlcolor={blue!80!black}
}

\usepackage{subcaption}
\usepackage{graphicx}
\usepackage{caption}
\title{ \raisebox{-0.2cm}{\includegraphics[width=1cm]{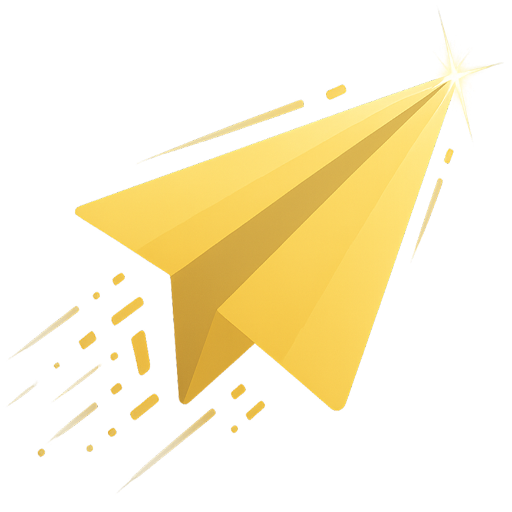}} Inference Time Context Sparsity: Illusion or Opportunity?}

\date{}

\author{
Sahil Joshi$^{\alpha,*}$ \quad
Prithvi Dixit$^{\beta,*}$ \quad
Agniva Chowdhury$^{\alpha}$ \quad
Anshumali Shrivastava$^{\alpha}$ \\[0.3em]
Joseph E. Gonzalez$^{\beta}$ \quad
Ion Stoica$^{\beta}$ \quad
Kumar Krishna Agrawal$^{\beta,\dagger}$ \quad
Aditya Desai$^{\gamma,\dagger}$ \\[0.8em]
$^{\alpha}$ Rice University \quad
$^{\beta}$ UC Berkeley \quad
$^{\gamma}$ IIT Bombay
}

\usepackage{bibentry}

\RequirePackage{amsmath}
\RequirePackage{amsthm}
\RequirePackage{amssymb}

\input{commands}

\usepackage{titlesec}
\titlespacing*{\paragraph}{0pt}{1.2ex plus 0.3ex minus 0.2ex}{0.7em}
\titlespacing*{\section}{0pt}{1.4ex plus 0.4ex minus 0.3ex}{1ex plus 0.2ex}
\titlespacing*{\subsection}{0pt}{1.2ex plus 0.3ex minus 0.2ex}{0.8ex plus 0.2ex}
\begin{document}


\begin{tcolorbox}[
    enhanced,
    colback=boxgold,
    colframe=boxgold,
    boxrule=0pt,
    arc=6pt,
    width=\textwidth,
    left=2pt,right=2pt,top=2pt,bottom=10pt
]


\maketitle

\begin{abstract}
    \input{sections_arxiv/abstract}

\end{abstract}
\end{tcolorbox}

\renewcommand{\thefootnote}{\fnsymbol{footnote}}%
\footnotetext[1]{Equal contribution.}%
\footnotetext[2]{Co-lead: \texttt{kagrawal@berkeley.edu}, \texttt{apdesai@cse.iitb.ac.in}.}%
\renewcommand{\thefootnote}{\arabic{footnote}}%
\setcounter{footnote}{0}%

\section{Introduction}
\input{sections_arxiv/introduction}


\section{Why Dense Attention is not meant for Long Context} \label{sec:theory}
\input{sections_arxiv/theory}

\section{Emergent Sparsity Observed in State-of-the-art Models}
\label{sec:emergent}
\input{sections_arxiv/emergent_sparsity}

\section{Sparsity and Hardware}

\input{sections_arxiv/kernels}

\label{sec:sparsity_hardware}

\section{Conclusion}

\input{sections_arxiv/conclusion}


\bibliographystyle{plain}
{
	\bibliography{bibliography}
}

\appendix
\input{sections_arxiv/theory_app}
\newpage

\end{document}

%% file: commands.tex
\ifx\BlackBox\undefined
\newcommand{\BlackBox}{\rule{1.5ex}{1.5ex}}  
\fi

\ifx\QED\undefined
\def\QED{~\rule[-1pt]{5pt}{5pt}\par\medskip}
\fi

\ifx\proof\undefined
\newenvironment{proof}{\par\noindent{\bf Proof\ }}{\hfill\BlackBox\\[2mm]}
\fi

\ifx\theorem\undefined
\newtheorem{theorem}{Theorem}
\fi

\ifx\example\undefined
\newtheorem{example}{Example}
\fi

\ifx\property\undefined
\newtheorem{property}{Property}
\fi

\ifx\lemma\undefined
\newtheorem{lemma}[theorem]{Lemma}
\fi

\ifx\corollary\undefined
\newtheorem{corollary}[theorem]{Corollary}
\fi

\ifx\proposition\undefined
\newtheorem{proposition}[theorem]{Proposition}
\fi

\ifx\remark\undefined
\newtheorem{remark}[theorem]{Remark}
\fi

\ifx\corollary\undefined
\newtheorem{corollary}[theorem]{Corollary}
\fi

\ifx\definition\undefined
\newtheorem{definition}{Definition}
\fi

\ifx\conjecture\undefined
\newtheorem{conjecture}[theorem]{Conjecture}
\fi

\ifx\axiom\undefined
\newtheorem{axiom}[theorem]{Axiom}
\fi

\ifx\claim\undefined
\newtheorem{claim}[theorem]{Claim}
\fi

\ifx\assumption\undefined
\newtheorem{assumption}[theorem]{Assumption}
\fi



%% file: sections_arxiv/abstract.tex
\noindent
Sparsity has long been a central theme in LLM efficiency, but its role in context processing remains unresolved. As LLM workloads shift toward longer contexts and agentic interactions, the compute and memory bottlenecks of attention become increasingly critical, raising the question of whether these constraints are fundamental.
Our position is that these constraints are artificial and unnecessary, and that the future of LLM inference lies in extreme but principled sparsity along the context dimension. This position is supported by several strands of empirical and theoretical evidence.
First, we find the insistence on dense attention unreasonable, since in  a long context a query effectively projects O(N) attention information into a hidden space of dimension $d \ll N$, making the process inherently lossy.
Second, we perform an extensive study of sparsity in LLMs spanning 20 models across five model families, varying context lengths, and different sparsity levels. We empirically demonstrate a strong trend: current LLMs, despite not being trained for context sparsity, are remarkably robust to inference-time decode sparsity across tasks of varying complexity, including retrieval, multi-hop QA, mathematical reasoning, and agentic coding. 
For instance, Qwen3.5-27B can tolerate up to 100× sparsity on benchmarks such as  RULER-HARD and AIME2025 without loss of quality, and up to 50× sparsity on LOFT and SWE with only a small drop in performance. These results suggest that a transition to complete sparsity may be possible without meaningful loss of capability.
Importantly, we also show that current hardware is already sufficient to realize substantial gains from this sparsity. For example, our sparse decode kernels accelerate large-context processing by up to 10× over FlashInfer at 50× sparsity levels on hardware such as the H100.
Overall, these results position extreme context sparsity not as a heuristic, but as a principled foundation for LLM inference, training, and architecture design: one that is both feasible and beneficial, and a compelling direction for future systems. 
\vspace{0.2cm}

\textbf{Code:} \url{https://github.com/skylight-org/sparse-attention-hub}

\textbf{Project page:} \url{https://sky-light.eecs.berkeley.edu}

%% file: sections_arxiv/introduction.tex
\input{figures_arxiv/teaser}

Sparsity in LLM inference has long been a goal for researchers. Efforts to sparsify the FFN component of Transformers \cite{ chen2020slide, fedusswitch2022, liu2023deja, shazeer2017} have largely converged on Mixture-of-Experts (MoE) as the de facto architecture across frontier labs \cite{agarwal2025gpt, liu2024deepseek, singh2025openai, yang2025qwen3,  zeng2026glm}. In contrast, a similar consensus has yet to emerge for attention mechanisms, or more broadly, context processors. Research on sparsity in attention dates back several years and has recently re-emerged in the LLM era, primarily focusing on emergent sparsity in trained models~\cite{achiam2023gpt, chen2025magicpig, desai2026vattention, desai2025hashattention, hooper2024fast, joshi2026socket, lai2025flexprefill, liu2026retrievalattention, liu2024scissorhands, tang2024quest, yang2024posttraining, zhang2025pqcache, zhang2025spargeattention, zhang2023ho}. However, such sparsity is rarely adopted in state-of-the-art inference engines, highlighting a limited understanding of its practical value. There has been some consolidation of this line of work through the adoption of sparse attention during post-training~\cite{deepseek2025v32, zeng2026glm}. However, demonstrated gains are largely confined to extremely large model regimes, raising questions about their general applicability. In parallel, an alternative form of sparsity has emerged through the development of lightweight context processors, such as linear attention~\cite{choromanski2021rethinking, joshi2026race, katharopoulos2020transformers, peng2021random, zaheer2020bigbird}, and SSMs~\cite{gu2024mamba, lahoti2026mamba, ye2025longmamba, yu2026blockbiased}. While these approaches have seen some adoption, modern architectures still retain full scaled dot-product attention (SDPA) layers, underscoring the limited expressivity of purely SSM-based models. This leads to a fundamental question: are the quadratic compute bottleneck during prefill and the linear memory bottleneck during decoding inherent constraints that are here to stay?
\noindent
This question is becoming increasingly important as LLM workloads shift toward substantially longer contexts and generations. Emerging use cases such as agentic tool use, code generation, retrieval augmented generation, multi document reasoning, long form dialogue, and repository scale software understanding are all driving this trend~\cite{guu2020realm,   hou2024llm4se, izacard2021fid, izacard2022atlas, li2022alphacode, press2022selfask, roziere2023codellama,  ruan2022survey,   schick2023toolformer, touvron2023llama, yao2022react}. To ground this in concrete terms, a single 50 page PDF can contain on the order of 33K tokens. An LLM that aims to condition its responses on corporate or legal documents may therefore need to process hundreds of thousands, if not millions, of tokens within a single context window. At the same time, expectations for state of the art models continue to increase, with each generation pushing toward larger and more persistent context handling. Recent examples further highlight this shift. Systems such as openclaw~\cite{openclaw2026} have been reported to use system prompts reaching 160K tokens\footnote{https://github.com/openclaw/openclaw/issues/21999}, illustrating how rapidly context sizes are expanding in practice. Importantly, these workloads do not merely involve longer prompts, but also often require long form generation over extended histories. This challenge is further amplified in agentic settings, where models iteratively generate outputs, incorporate new information, and repeatedly condition on growing interaction traces among agents. As a result, the question becomes of utmost importance: what does the future of LLM Inference looks like in context dimension?

\noindent
We take the following position: \textbf{\textit{The future of LLM inference lies in extreme sparsity along the context dimension.}} Our position is primarily motivated by our large-scale empirical findings on inference-time sparsity, as well as the substantial speedups we demonstrate over FlashInfer even under sparsity patterns that are arguably highly irregular.

\paragraph{Sparsity already exists in new generation models:} In Section~\ref{sec:emergent} we observe empirically that sparsity is already emerging in modern architectures. Larger and more recent models exhibit remarkable robustness to aggressive context sparsification across a wide spectrum of tasks as shown in Section~\ref{sec:emergent}. This holds for relatively simpler benchmarks such as RULER-HARD, a challenging subset of RULER~\cite{hsieh2024ruler}, and LOFT~\cite{Lee2024Loft}, as well as for more complex reasoning tasks such as AIME~\cite{dekoninck2026matharena} and real-world agentic workloads such as SWE~\cite{jimenez2023swebench} (50+ agentic turns). To our knowledge, this is the first study to examine inference-time sparsity on an agentic workload. Of course, inducing sparsity during training is ideal, but the emergence of sparsity at inference time reinforces the possibility that models can operate effectively in highly sparse regimes and provides greater confidence in designing training procedures that explicitly target sparsity.
\noindent
We empirically evaluate five model families, including Llama3\cite{dubey2024llama}, Qwen2.5\cite{qwen2025qwen25technicalreport}, Qwen3.5\cite{qwen35blog}, Gemma3\cite{gemmateam2025gemma3technicalreport}, and Ministral3\cite{liu2026ministral}, across four tasks: RULER-HARD (32K), LOFT (32K and 128K), AIME (65K generation length), and SWE (50+ agentic turns). Our results show that the effectiveness of inference-time sparsity improves with both model scale and the use of hybrid architectures. In particular, larger models from newer-generation families such as Qwen3.5, Gemma3, and Ministral3 sustain quality parity with dense execution even at 50$\times$ sparsity. For smaller standard models, the observed degradation can largely be mitigated through stochastic index selection \cite{desai2026vattention}. Notably, these gains are achieved purely at inference time, suggesting that incorporating sparsity during training could yield even greater benefits.

\paragraph{Sparsity can be leveraged for system gains:}

The alignment of sparsity with hardware is of paramount importance for fully realizing its benefits. It is therefore essential to evaluate whether sparsity actually alleviates the underlying bottlenecks. A common belief is that efficiency gains require block-structured sparsity~\cite{zhu2025an}; without it, sparsity is unlikely to translate into real speedups. We present an argument against this view.
\noindent
Notably, DeepSeek Attention~\cite{deepseek2025v32} demonstrates both training- and inference-time speedups with token-level sparsity during prefill. This serves as evidence that token-level sparsity can indeed translate into practical efficiency gains without imposing excessive structure on the model. We extend this further to decoding. For decoding, we provide kernels that accelerate inference under per-token, per-query, and per-head sparsity, i.e., highly irregular sparsity patterns. These gains persist even under grouped query attention, where the number of query heads can be a factor (typically 4) larger than the number of key-value heads. This is possible because the KV cache’s vector dimension provides sufficient contiguous memory to make such sparsity effective on modern hardware. In particular, our optimized sparse attention kernel, built on top of FlashInfer with a paged KV-cache backend, achieves up to $10\times$ kernel speedup at $50\times$ sparsity for large batch sizes.

\paragraph{Non-existence of a truly dense attention:} 
Apart from the empirically strong results, our position and research is strongly rooted in the idea that dense attention is incompatible with long context. We show a simple result that truly dense attention does not exist in practice: full attention is ultimately bottlenecked by the hidden dimension, causing it to collapse unable to distinguish between varying attention distributions. While the result itself is straightforward, the implications are significant. It suggests that complete sparsity is not merely a practical approximation but principally a superior objective.
\noindent
Overall, this paper argues that the community should actively explore extreme sparsity along the context dimension without compromises such as partially retaining full attention layers. The empirical emergence of sparsity across multiple axes highlights strong potential, and decode-time sparse kernel analyses in Section~\ref{sec:emergent}, augmented by results from recent DeepSeek models, demonstrate the significant efficiency gains that can be unlocked.

%% file: figures_arxiv/teaser.tex
\begin{figure}[t]
    \centering
    \begin{subfigure}{0.55\textwidth}
        \centering
        \includegraphics[width=\linewidth]{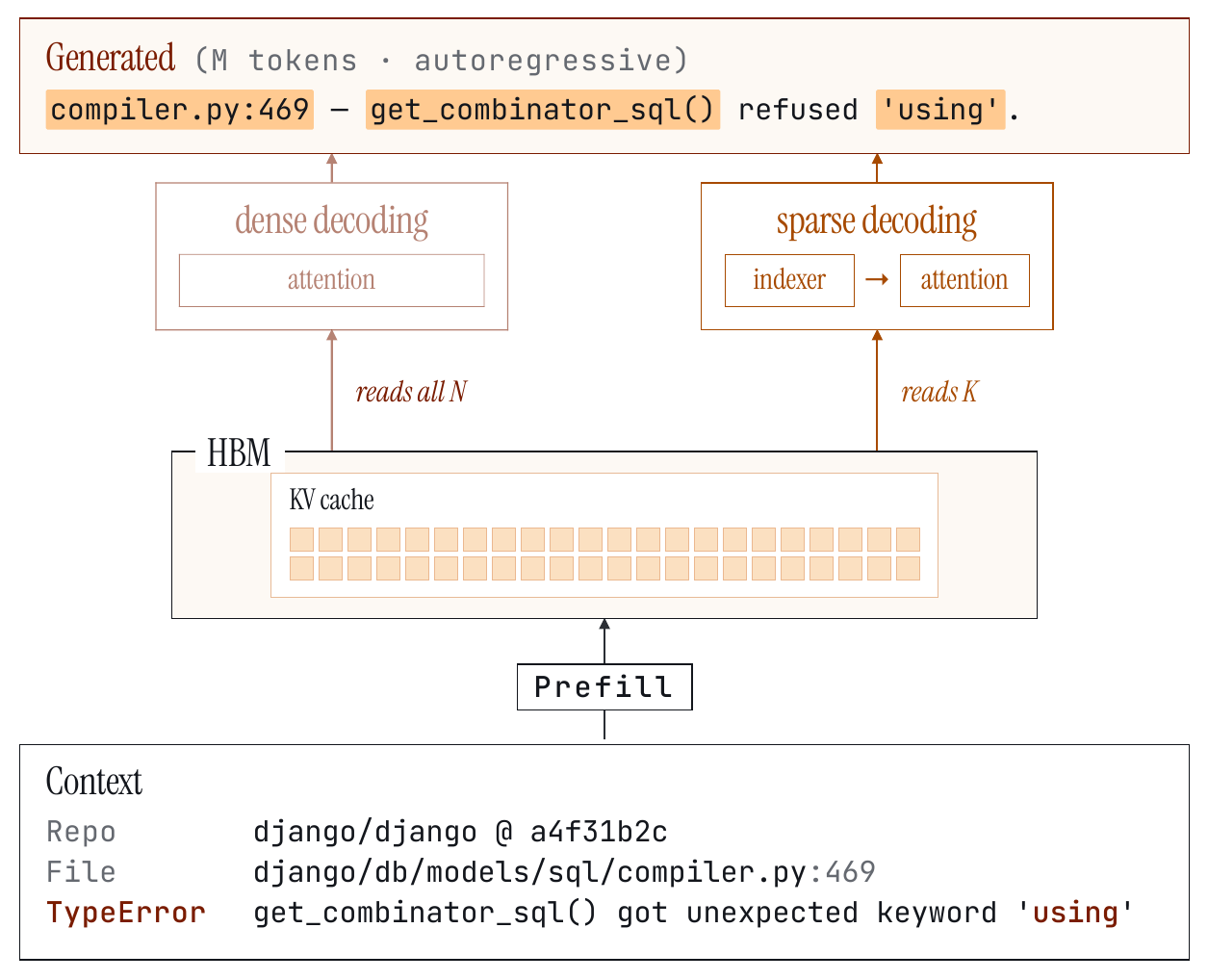}
        \caption*{\footnotesize\textbf{(a)} Inference-time attention I/O across decode regimes.}
    \end{subfigure}%
    \begin{subfigure}{0.45\textwidth}
        \centering
        \includegraphics[width=\linewidth]{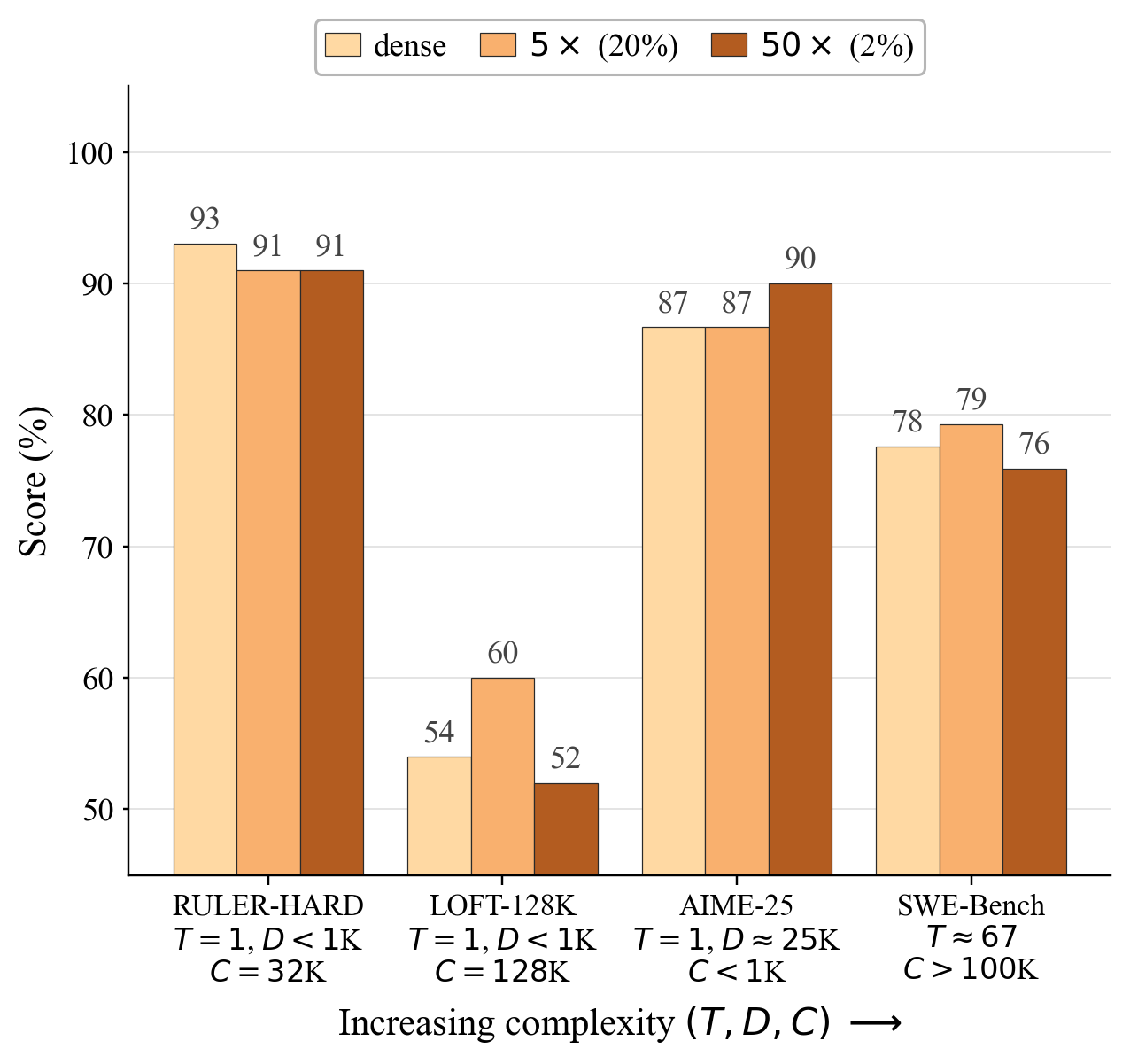}
        \caption*{\footnotesize\textbf{(b)} Qwen3.5-27B across workloads.}
    \end{subfigure}
    \caption{\textbf{Inference-time $50\times$ context sparsity is bandwidth-friendly by construction (a) and retains near-dense quality across diverse workloads on a single model (b).}
    \textbf{(a)} Three decode regimes share an HBM band but read from it differently. \emph{Dense} reads the full KV cache every step ($\mathcal{O}(N{\cdot}d)$ bytes); \emph{Sparse} routes through a lightweight indexer that selects $k$ rows ($\mathcal{O}(k{\cdot}d)$ bytes, $k{\ll}N$); \emph{Linear} (Gated DeltaNet) reads a fixed-size recurrent state $S$ ($\mathcal{O}(d^{2})$ bytes, constant in $N$). All three are memory-bandwidth bound on H100/B200; the contrast is whether per-step traffic scales with context length.
    \textbf{(b)} Qwen3.5-27B on four workloads ordered by increasing complexity along $(T,D,C)$ -- turns, decoded tokens per turn, input context per turn (\S\ref{sec:emergent}). Configurations: \textbf{RULER-HARD-32K} $(T{=}1,\ D{<}1$K$,\ C{=}32$K$)$; \textbf{LOFT-128K} $(T{=}1,\ D{<}1$K$,\ C{=}128$K$)$; \textbf{AIME-2025} $(T{=}1,\ D{\approx}25$K$,\ C{<}1$K$)$; \textbf{SWE-Bench Django} $(T{\approx}67,\ D{\sim}1$K per turn$,\ C$ grows to ${>}100$K$)$. Sparse within $\leq\!2$ points of dense at $50\times$ on retrieval (RULER, LOFT), reasoning (AIME), and agentic coding (SWE-Bench $S_3$ subset, $n{=}58$, App.~\ref{app:swe}).
    Kernel-level speedup on H100/B200 is deferred to \S\ref{sec:sparsity_hardware} and Tab.~\ref{tab:backend_speedup}.}
    \label{fig:teaser}
\end{figure}

%% file: sections_arxiv/theory.tex
We begin by examining whether attention layers can remain truly dense as context length grows, and show that they cannot. 

\subsection{Dense Attention Collapses Through The Hidden Dimension.}

While sparse attention and SSMs have their own shortcomings, fully dense attention has its own bottleneck: it assigns a weight to every context token, but the layer passes forward only a fixed-dimensional hidden vector.

\begin{theorem}
\label{thm:dense-attention-hidden-collapse}
Let \(V\in\mathbb{R}^{N\times d}\) be any value matrix, and let the dense attention output be
$o=V^\top a$,
where \(a=(a_1,\dots a_N)\) is an attention distribution over \(N\) context tokens. If \(d<N-1\), then this map is not injective on the attention simplex. In particular, there exist two distinct dense attention distributions \(a,a'\) such that
$a\neq a'$,
and
$V^\top a = V^\top a'$.
Thus, a \(d\)-dimensional post-attention embedding cannot preserve all dense attention-score variations over more than \(d+1\) context tokens.
\end{theorem}

\begin{corollary}
\label{cor:million-token-hidden-width}
If all dense attention distributions over \(N\) context tokens must remain distinguishable after the attention output \(V^\top a\), then the hidden dimension must satisfy \(d\ge N-1\). Thus, losslessly preserving arbitrary dense attention over a million-token context requires million-scale hidden width.
\end{corollary}
\noindent
This is the embedding bottleneck. When \(d \ll N\), the map \(a \mapsto V^\top a\) collapses many distinct attention patterns into the same hidden representation; hence not every fine-grained variation in a dense attention distribution can be carried forward by the post-attention vector. This conclusion is consistent with lower bounds for indexed lookup: finite-precision recurrent models, including RNNs, LSTMs, state-space models, and recurrent linear attention, require hidden-state size \(\Omega(N)\) to recover arbitrary tokens from an \(N\)-token sequence~\cite{bhattamishra2024separations}. These observations motivate a complete context-sparse attention.
\input{figures_arxiv/ruler1}

%% file: figures_arxiv/ruler1.tex
\begin{figure}[!t]
    \centering
    \includegraphics[width=\linewidth]{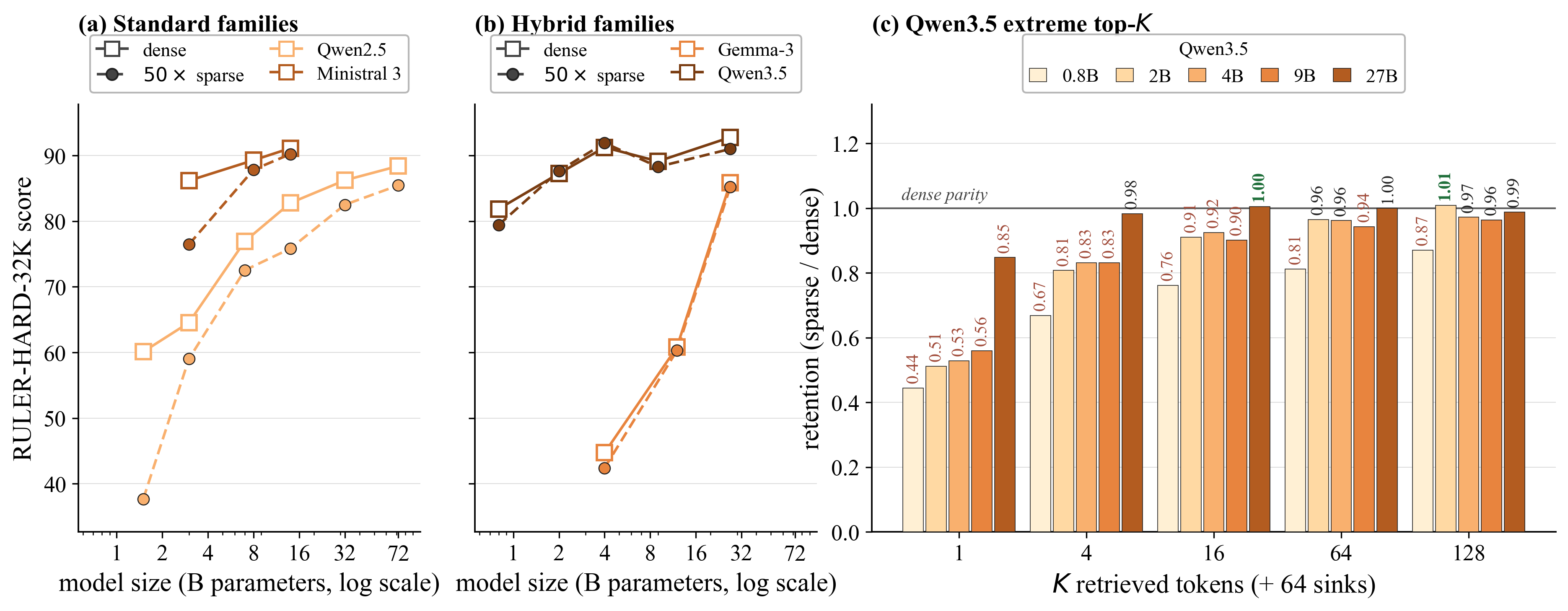}
    \caption{\textbf{RULER-HARD-32K score across families, scales, and sparsity.}
    Scores are absolute (averaged over the six RULER-HARD subtasks). Squares mark dense ($1\times$); filled circles mark sparse ($5\times$ and $50\times$ oracle top-$k$). Each line is one checkpoint; lighter shade within a family means smaller scale. Panels (a) and (b) share the x-axis.
    \textbf{(a)~Standard families.} Qwen2.5 (blue) and Ministral~3 (amber) fan out as sparsity grows: Qwen2.5-1.5B drops from $60$ at dense to $38$ at $50\times$, while Qwen2.5-72B holds within $3$~points across the full sweep.
    \textbf{(b)~Hybrid families.} Gemma-3 (green, sliding-window) and Qwen3.5 (red, linear-attention) stay essentially flat at every scale; the small-checkpoint penalty visible in (a) is absent.
    \textbf{(c)~Qwen3.5 saturation under extreme top-$K$.} Retention $=$ score$_K /$ score$_{\text{dense}}$, computed per subtask and then averaged. $K$ retrieved tokens with $64$ sinks always retained, $K\!\in\!\{1,4,16,64,128\}$. Larger Qwen3.5 scales saturate earlier (27B reaches near-parity by $K{=}4$); 0.8B is still climbing at $K{=}128$ (0.87 at the right edge)~--- the knee migrates with scale.
    Three Qwen3.5-27B cells in the parsed CSV were unevaluated and imputed at ceiling continuation of adjacent rates (\texttt{nm3@5$\times$}, \texttt{nm2@10$\times$}~$\to 100$; \texttt{qa2@20$\times$}~$\to 63$); effect on the 6-subtask average is $<\!1$ point. LLaMA-3 omitted: no per-subtask RULER-HARD data in the current eval set.}
    \label{fig:ruler1}
\end{figure}

%% file: sections_arxiv/emergent_sparsity.tex
We now turn to empirical evidence showing that even under current training recipes and latest architectures, where attention is not explicitly trained to be sparse, sparsity emerges regardless. This reinforces our proposal: it suggests that in principle, models can transition to extremely sparse context processing without any loss in capability. We adopt the following evaluation setup.
\\
\noindent
\textbf{Algorithms:} We enforce sparse attention at decode time and measure the effect across models, architectures, and tasks. The primary mechanism is exact \emph{oracle} top-$k$ selection (to eliminate confounds from approximate top-$k$ indexers); in selected experiments we also report stochastic indexing via vAttention~\cite{desai2026vattention}.
\\
\noindent
\textbf{Datasets:} For long-context retrieval, we report the average score on RULER-32K-HARD, a challenging subset consisting of six tasks from RULER-32K~\cite{hsieh2024ruler} (\texttt{fwe}, \texttt{qa1}, \texttt{qa2}, \texttt{vt}, \texttt{nm2}, and \texttt{nm3}). On the LOFT~\cite{Lee2024Loft} benchmark, we evaluate long-context performance by averaging results over five datasets (\texttt{hotpotqa}, \texttt{nq}, \texttt{musique}, \texttt{qampari}, and \texttt{quest}) at both 32K and 128K context lengths. To assess sparsity in long-form generation settings, we evaluate the largest hybrid model, Qwen3.5-27B~\cite{qwen35blog}, on AIME2025~\cite{dekoninck2026matharena}, where we limit the generation length to 65K tokens. For agentic workloads, we evaluate Qwen3.5-27B on the SWE-Bench Django~\cite{jimenez2023swebench} subset ($114$ tasks, $\le 100$ turns per task) at three sparsity levels: dense, $5\times$ sparsity, and $50\times$ sparsity. To our knowledge this is the first inference-time sparsity study on an agentic benchmark.
\\
\noindent
\textbf{Models.} We investigate five families, with a total of 20 models: three standard transformers (Qwen2.5~\cite{qwen2025qwen25technicalreport}, Ministral3~\cite{liu2026ministral}, and Llama3~\cite{dubey2024llama}) and two hybrids (Qwen3.5~\cite{qwen35blog} and Gemma3\cite{gemmateam2025gemma3technicalreport}), enabling a controlled comparison of sparsity behavior across architectural variants.

\input{figures_arxiv/ruler3}

\paragraph{RULER-32K(Hard Subset)} 
On the hard subset of RULER, we observe two consistent trends across model families in Figure~\ref{fig:ruler1}. First, hybrid architectures such as Qwen3.5 and Gemma3 exhibit significantly greater robustness to context sparsity, maintaining performance even at up to 50× sparsity with little to no degradation in quality. Interestingly, this robustness appears largely invariant to model scale within these families, suggesting that the inclusion of SSM or linear-attention layers may inherently improve resilience to sparse context retrieval.
\\
\noindent
We further evaluate the Qwen3.5 (See Figure~\ref{fig:ruler1}) family of hybrid models under extreme context sparsity, where the model is restricted to using only a very small number of retrieved tokens (1–128 tokens). Even in these highly constrained regimes, larger hybrid models exhibit striking robustness to sparsity. In particular, restricting attention to just 128 retrieved tokens corresponds to approximately 250× sparsity while still preserving strong performance. Given the broader trajectory of state-of-the-art LLMs toward larger parameter scales and increasingly hybrid architectures, these findings suggest that future inference systems may rely more heavily on extremely sparse context processing rather than dense attention across the entire context window.
\\
\noindent
Second, for standard dense-attention architectures such as Llama3, Qwen2.5, and Ministral3, robustness to sparsity improves steadily with model size as seen in Figure~\ref{fig:ruler1}. In particular, the largest models in these families are able to preserve quality even under 50× sparsity. This suggests that scaling alone may enable stronger implicit retrieval and context localization capabilities, even in the absence of explicit architectural mechanisms for sparse processing.
\noindent
For smaller standard models, the performance of top-$k$ can be significantly lower than dense model (see Figure~\ref{fig:ruler1}). A common explanation for the failure of top-$k$ sparse attention is that attention mass is often diffusely distributed across the context, making deterministic selection of small number of tokens ineffective. However, this does not necessarily imply that attention itself must remain dense; rather, it points to the need for more effective context selection mechanisms. As shown in Figure~\ref{fig:ruler3}, we find that stochastic index-selection approaches such as vAttention are able to nearly recover the quality of full attention while still operating under extreme sparsity in the context dimension (up to 50× sparsity in our experiments). These results suggest that the core limitation of conventional top-$k$ sparsification may stem less from sparsity itself and more from the deterministic and locality-biased nature of the selection process.

\input{figures_arxiv/loft}

\paragraph{LOFT-32K and LOFT-128K}

To evaluate context sparsity beyond synthetic retrieval settings, we additionally study LOFT, a more natural retrieval and question-answering benchmark. LOFT appears to be substantially more challenging than RULER-HARD, with even the dense baseline models achieving relatively modest performance. Despite this increased difficulty, the qualitative trends with respect to sparsity remain largely consistent with those observed on RULER-HARD.
\\
\noindent
To further study the effect of sequence length, we evaluate LOFT under two context regimes: 32K and 128K tokens. Across both settings, robustness to sparsity remains remarkably stable. Intuitively, one might expect longer contexts to naturally induce greater effective sparsity, especially for retrieval-oriented tasks where only a small subset of tokens should be relevant to the query. However, we do not observe a corresponding increase in inference-time sparsity as context length grows. This suggests that current models may not automatically adapt their retrieval behavior with increasing context size, pointing to potential opportunities during training to explicitly encourage sparsity to scale with context length.

\paragraph{AIME25}
We use AIME 2025 to evaluate the effect of context sparsity on long-form autoregressive generation. While we permit generations of up to 65K tokens, models produce roughly 25K tokens on average across evaluation samples, making this setting particularly sensitive to accumulated approximation errors. Sparse attention introduces perturbations at every embedding update, raising the concern that such errors may compound not only across layers, but also across thousands of autoregressively generated tokens.
\\
\noindent
Despite these concerns, the results on AIME demonstrate remarkable robustness to aggressive sparsification. Even over extremely long generations, model quality remains largely stable, suggesting that sparsity-induced approximation errors do not significantly destabilize long-horizon reasoning or autoregressive decoding dynamics. Furthermore, we observe that increasing sparsity results in only a marginal increase in the average number of generated tokens required to solve a task. This is an important practical observation: it indicates that improvements in per-token decoding efficiency can translate into genuine end-to-end reductions in task completion time, rather than being offset by substantially longer generations.

\begin{figure}
    \centering
    \begin{subfigure}{0.5\textwidth}
        \centering
        \includegraphics[width=\linewidth]{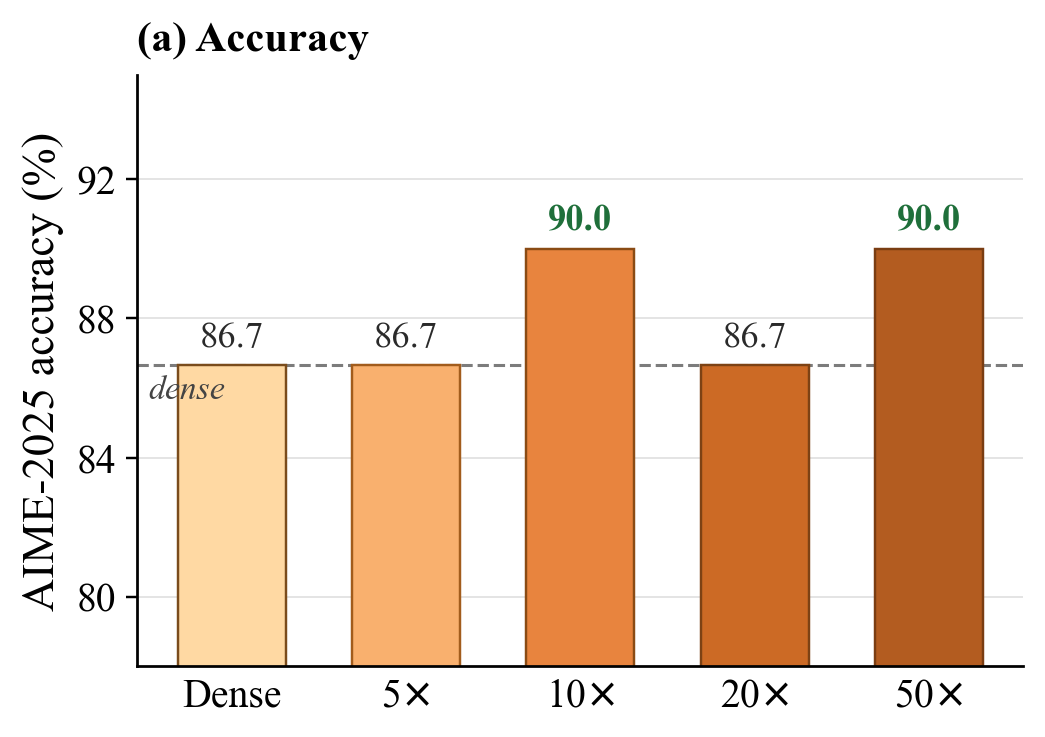}
    \end{subfigure}\hfill
    \begin{subfigure}{0.5\textwidth}
        \centering
        \includegraphics[width=\linewidth]{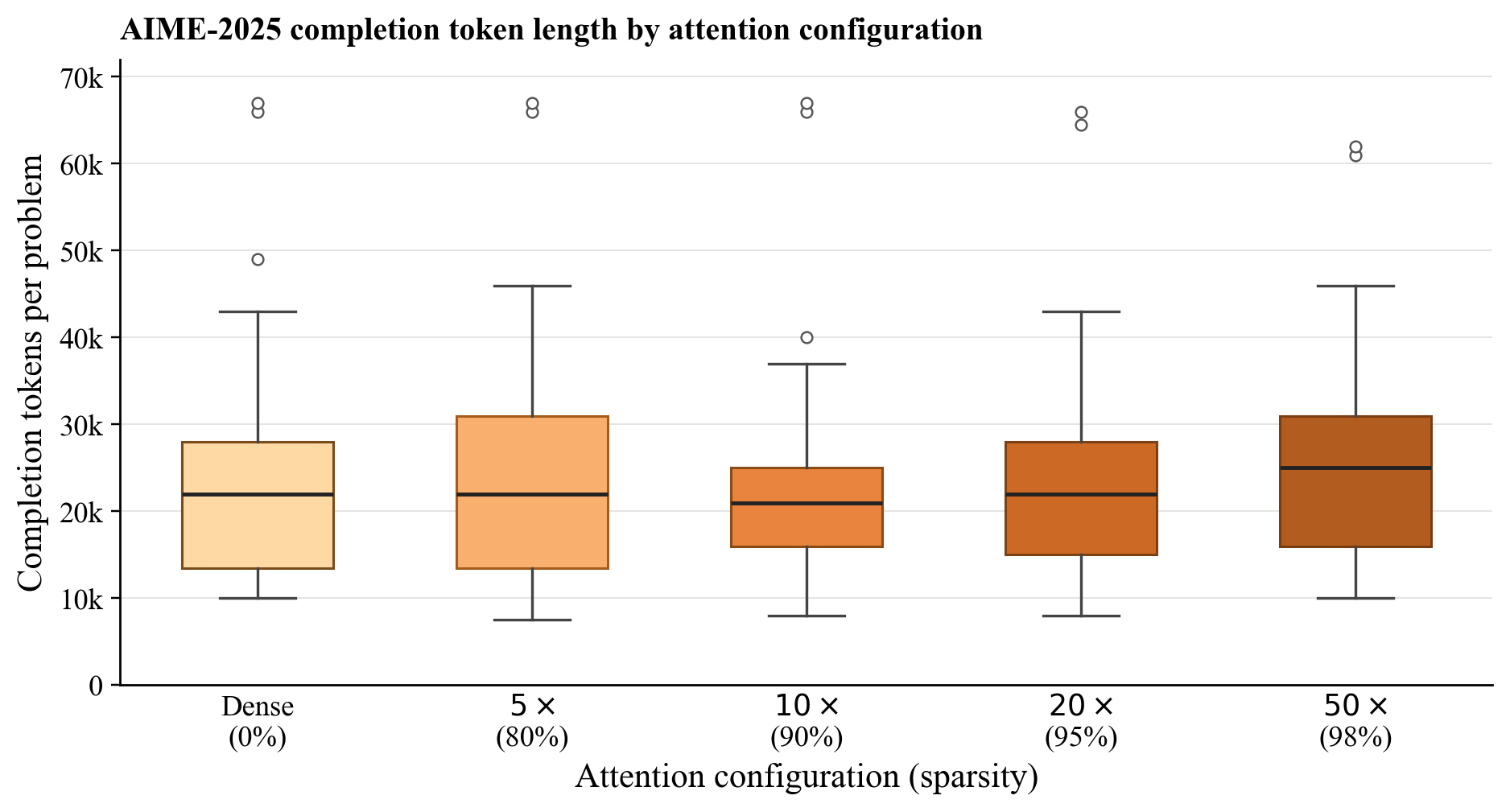}
    \end{subfigure}
    \caption{We use AIME 2025 for evaluating long-form generation. Although we allow generations up to 65K tokens, models generate approximately 25K tokens on average across samples. Since sparse attention introduces approximation errors at each embedding update, an important concern is whether these errors compound across layers and, more critically, across autoregressively generated tokens. The results on AIME are therefore particularly promising: they indicate that models remain robust to sparsity-induced errors even over very long generations, suggesting that aggressive context sparsification does not necessarily destabilize long-horizon reasoning or autoregressive decoding. Additionally, we observe that increasing sparsity leads to only a slight increase in the average number of generated tokens per task. This suggests that the gains observed in decoding speed can translate into meaningful end-to-end task completion speedups in practice.}
    \label{fig:aime}
\end{figure}


\paragraph{SWE-Bench (Django)}
\input{figures_arxiv/swe}
SWE-bench~\cite{jimenez2023swebench} is a benchmark of $2{,}294$ real-world GitHub issues drawn from $12$ popular Python repositories. Each task instance pairs an \emph{issue} (a bug report or feature request, $195$ words on average) with the corresponding repository snapshot at the base commit (codebases average $3{,}010$ non-test files and $438$K lines of code) and a set of \emph{fail-to-pass} tests; the task is to produce a patch that resolves the issue and passes the tests. Gold patches average $32.8$ lines spread over $1.7$ files and $3$ functions, so each instance reduces to identifying a small set of edits within a large body of context. We evaluate on the $114$-task Django subset of SWE-Bench Lite, the curated variant of the original benchmark that biases toward single-file gold patches.
\\
\noindent
\textbf{Agentic compute profile:} We run Qwen3.5-27B as a tool-using coding agent inside the \texttt{mini-swe-agent} harness, with a limit of $250$ turns per task. At every step, the agent sees its complete conversation history (issue text, prior tool calls, and their stdout/stderr), emits one new tool invocation, and the harness appends the tool output to the next prompt. Consequently, the effective context length grows monotonically from $\sim$$1$K tokens of issue text and repository metadata at turn one, to $\sim$$20$K tokens by turn $60$ on a typical resolved trajectory, and exceeds $100$K tokens on the tail of trajectories that hit the step limit.  We compare three densities on the Django subset: dense, $5\times$ (20\%) sparsity, and $50\times$ ($2\%$) sparsity. To our knowledge this is the first inference-time sparsity study on an agentic benchmark.
\\
\noindent
\textbf{Resolution rate:} A clean three-way comparison requires controlling for runtime errors (e.g., timeouts, server errors) that can bury an otherwise competent run as an empty patch. We report two head-to-head subsets (Figure~\ref{fig:swe}, left): the strict $n{=}58$ subset where all three configurations produced a valid evaluation verdict, and a broader $n{=}70$ subset that also admits tasks where the agent exhausted its $250$-turn budget (step-limit hits, an agent-side failure mode shared across configs; see Appendix~\ref{app:swe}). On the valid-eval subset the three configurations resolve within $\sim$$2$ points of each other ($77.6\%$ dense, $79.3\%$ at $5\times$, $75.9\%$ at $50\times$): on tasks that all three configurations finish cleanly, sparse matches dense. On the step-limit-inclusive subset the gap widens to $\sim$$4$ points ($71.4\%$ dense vs $67.1\%$ at both sparse settings) because sparse runs collapse into degenerate command loops slightly more often, exhausting the turn budget without submitting. The full subgroup analysis is in Appendix~\ref{app:swe}.
\\
\noindent
\textbf{Why we report a strict subset (error analysis):} A non-trivial fraction of trajectories never enter the patch-quality evaluation because the agent never emits a parseable diff (\texttt{empty\_patch}), or emits one the harness cannot apply (\texttt{error}). These are not attention-quality signal, so we attribute them to a root cause before deciding what to drop. Dense's $13$ \texttt{empty\_patch} cases are $9$ \texttt{CalledProcessError} (the per-instance \texttt{docker run} returned exit $125/127$ \emph{before} the agent could start: a daemon-side launch failure, not a model failure) plus $4$ \texttt{LimitsExceeded} (the agent ran but hit the $250$-turn cap without submitting). Sparse \texttt{empty\_patch} is bigger ($34$ at $5\times$, $26$ at $50\times$) and has a different root cause: it is dominated by \texttt{InternalServerError} ($17$ and $11$, respectively) and \texttt{Timeout} ($2$ and $4$), both of which are the vLLM server crashing or stalling mid-trajectory under the sparse-attention-hub backend, an engineering instability of the serving stack rather than a property of sparse attention itself. The remaining sparse \texttt{empty\_patch} cases are \texttt{LimitsExceeded} ($11$ and $11$); $50\times$ additionally has $3$ \texttt{error} cases where the produced patch was rejected by \texttt{git apply}. Because these failure modes either pre-date the model call (\texttt{CalledProcessError}) or are mid-call infrastructure crashes (\texttt{InternalServerError}, \texttt{Timeout}), they bury an otherwise competent run as a zero, and we drop them before comparing patch quality. The $n{=}58$ subset is the strict residual where all three configurations emitted a non-empty patch that the harness scored; the $n{=}70$ subset additionally admits \texttt{LimitsExceeded} on the rationale that turn-budget exhaustion is a shared, model-side failure mode. The full per-config outcome$\times$exit-status breakdown is in Appendix~\ref{app:swe}.

\input{figures/table_sparse_backend}

\noindent
Together, these experiments probe sparsity across five key axes, yielding the following generic takeaways:

\begin{itemize}[leftmargin=*]

\item \textbf{Scale:} Performance under top-$k$ sparsity improves with model size; the gap between sparse and dense decoding largely closes at scale.

\item \textbf{Architecture:} Hybrid models tolerate sparsity better than standard transformers. Notably, at high sparsity levels (e.g., $50\times$ reduction), hybrid models maintain strong performance. In fact for larger hybrid model Qwen3.5-27B, even 16-32 tokens in top-$k$ are enough to achieve parity with the dense model on RULER-HARD.

\item \textbf{Context Length:} The qualitative behavior of sparsity remains consistent across context lengths. Observations at 32K and 128K contexts largely align, indicating that sparsity properties generalize to long-context settings without introducing additional degradation. A case can be made for more sparsity in longer contexts and it seems like it would need training time modifications to achieve it. 

\item \textbf{Algorithm:} 
Generally Hybrid models and large models with standard architecture show good results with top-$k$ sparsity. However, for smaller models with standard architecture, top-$k$ sparsity may not suffice. In such a case, stochastic sparsity with methods such as vAttention can still enable sparsity in attention while maintaining quality.

\item \textbf{Task complexity:} Qwen3.5-27B robustness to sparsity holds across single-hop retrieval and multi-hop QA (RULER-HARD, LOFT), mathematical reasoning (AIME~2025), and agentic coding (SWE-Bench Django). This evaluation showcases that even inference time sparsity does not deteriorate the capability of models in many real world long context applications.

\end{itemize}
\input{figures/table_double_sparsty}

%% file: figures_arxiv/ruler3.tex
\begin{figure}[t]
    \centering
    \includegraphics[width=0.78\textwidth]{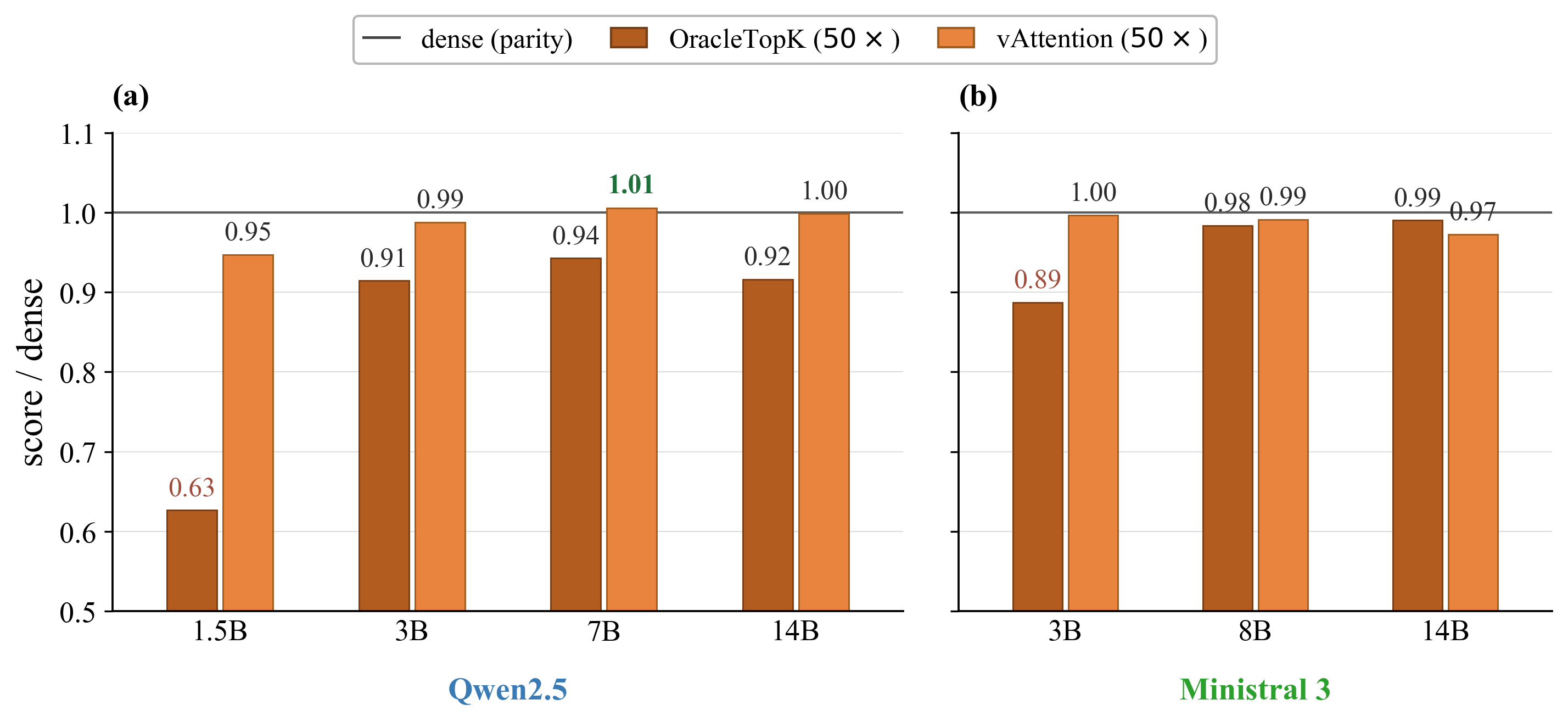}
    
    \caption{\textbf{vAttention vs OracleTopK retention at $50\times$ sparsity on RULER-HARD-32K.} Bars report relative score (sparse / dense); the horizontal line at $1.0$ marks dense parity. Values $\geq\!1.0$ are bold green; values $<\!0.90$ are muted red. Panels separate model family ((a) Qwen2.5, (b) Ministral 3). Traditionally, the failure of top-$k$ sparse attention has been attributed to the diffusion of attention scores across the context. However, this does not necessarily imply that attention must remain dense; rather, it suggests the need for a different mechanism for selecting relevant context. We observe that stochastic index-selection (vAttention) closely tracks dense parity at every Qwen2.5 and Ministral checkpoint, while deterministic OracleTopK collapses on smaller standard models (Qwen2.5-1.5B drops to $\sim$$0.63$ of dense). This suggests that the primary limitation of conventional top-$k$ sparsification lies not in sparsity itself, but in the determinism and locality of the selection mechanism.}
    \label{fig:ruler3}
\end{figure}

%% file: figures_arxiv/loft.tex
\begin{figure}[t]
    \centering
    \includegraphics[width=\textwidth]{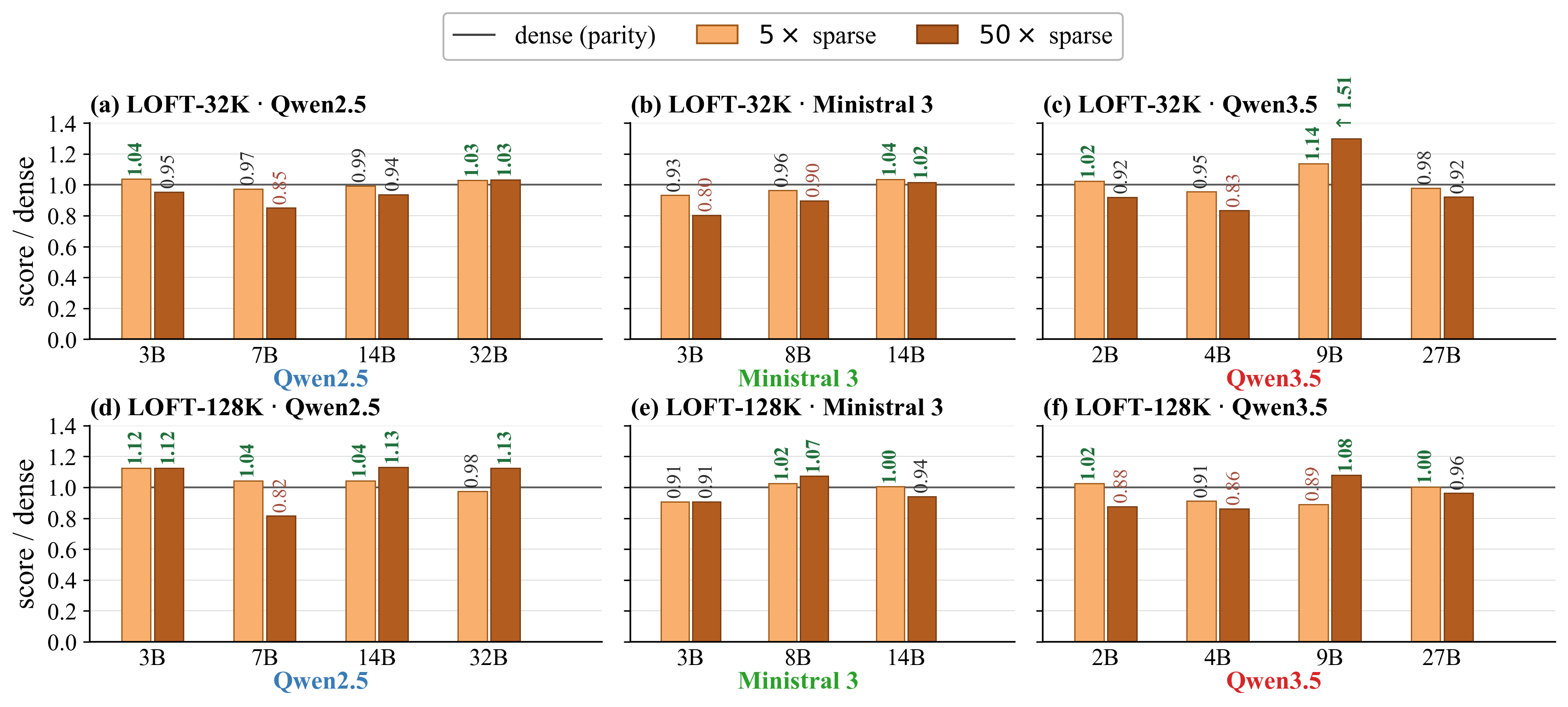}
    \caption{\textbf{LOFT subspan-EM retention under $5\times$ and $50\times$ inference-time sparsity.} Bars report relative score (sparse Subspan-EM / dense Subspan-EM); the horizontal line at $1.0$ marks dense parity. Values $\geq\!1.0$ (sparse meets or exceeds dense) are bold green; values $<\!0.90$ (over a 10\% relative drop) are muted red. Layout: rows separate context length (top: 32K, bottom: 128K); columns separate model family (Qwen2.5: 3B--32B, Ministral3: 3B--14B, Qwen3.5: 2B--27B). At $5\times$ sparsity, retention sits within $\sim$$2\%$ of dense across nearly every checkpoint; at $50\times$, mid-scale standard checkpoints (Qwen2.5-7B, Ministral-3B) suffer the largest drops while the hybrid Qwen3.5 family stays close to parity except at the smallest scales. Qwen3.5-9B at 32K shows a $1.51\times$ ratio --- a small-denominator effect on a low-scoring task --- consistent with the qualitative picture from RULER-HARD that sparsity does not degrade quality once a model is large enough.}
    \label{fig:loft}
\end{figure}

%% file: figures_arxiv/swe.tex
\begin{figure}[!t]
    \centering
    \includegraphics[width=0.95\textwidth]{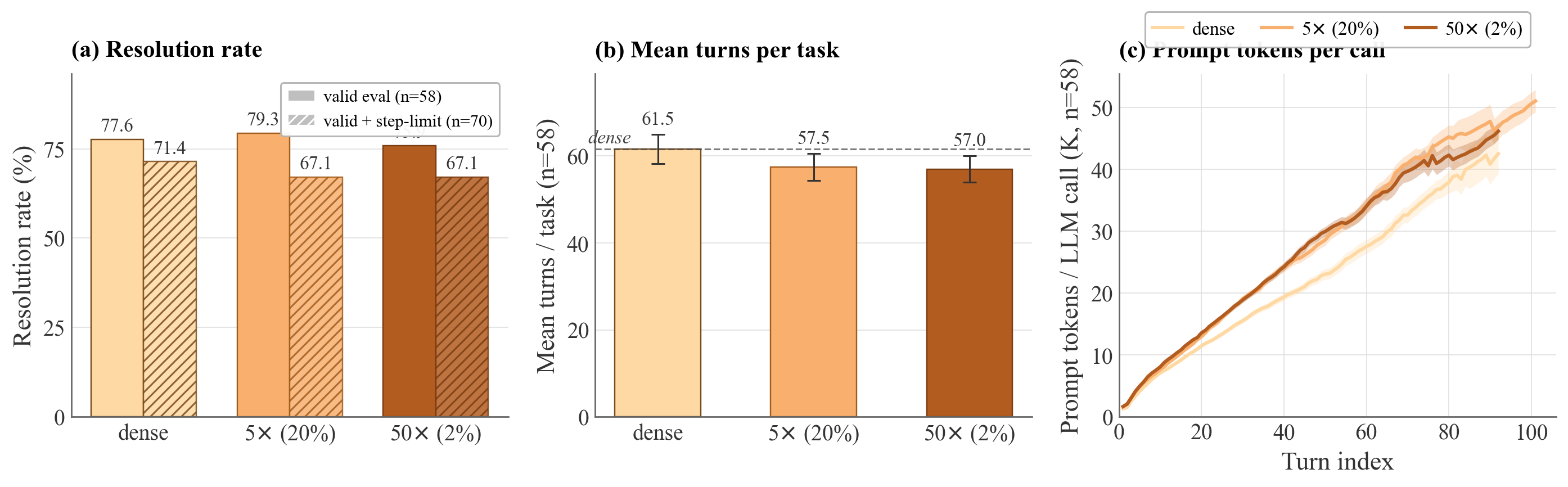}
    \caption{\textbf{SWE-Bench Django head-to-head, Qwen3.5-27B under dense, $5\times$, and $50\times$ sparsity.}
    Left: resolution rate on two nested subsets (solid: $n{=}58$ where all three produced a valid eval verdict; hatched: $n{=}70$ that also admits tasks where any config exhausted its $250$-turn budget). Sparse matches dense within $\sim$$2$ points on the strict subset; the $\sim$$4$-point gap on the broader subset is driven by sparse runs occasionally collapsing into degenerate command loops.
    Middle: mean agent turns per task on the strict $n{=}58$ subset (error bars: SEM).
    Right: mean prompt tokens per LLM call as a function of turn index, averaged across $n{=}58$ tasks that reached that turn (shaded band: $\pm$SEM; tail truncated where fewer than six tasks remain); the sparse curves track dense closely with a small ($\sim$$6\%$) per-call offset that grows with context length.
    Full subgroup table, per-task cost, and outcome composition (including the step-limit failure mode) are in Appendix~\ref{app:swe}.}
    \label{fig:swe}
\end{figure}

%% file: figures/table_sparse_backend.tex
\begin{table}[t]
\centering
\caption{\textbf{Per-query, per-head irregular sparse decode is up to $76\times$ faster than FlashInfer at extreme sparsity}
Speedup over FlashInfer~\cite{flashinfer2024} on H100 80GB HBM3 (FP16, GQA $H_q{=}32$, $H_{kv}{=}8$, $D{=}128$, page size 16, NHD, 128K context).
$S\times$ sparsity implies each query-head attends to $1/S$ fraction of total tokens. The Speed up of  $<\!1\times$ denotes overhead.
We generally break-even at $10\times$ sparsity for every batch with huge speedups of  $10$--$76\times$ in the $50$--$500\times$ regime.
Note that no block structure imposed on the sparsity.}
\label{tab:backend_speedup}
\setlength{\tabcolsep}{5pt}
\begin{tabular}{c r | r r r r r r r r}
\toprule
$B$ & FlashInfer (ms) & 2x & 4x & 10x & 20x & 50x & 100x & 200x & 500x \\
\midrule
1  & 0.19 & 0.32$\times$ & 0.63$\times$ & 1.45$\times$ & 2.58$\times$ &  5.57$\times$ & 10.25$\times$ & 11.05$\times$ & 11.14$\times$ \\
4  & 0.72 & 0.33$\times$ & 0.66$\times$ & 1.64$\times$ & 3.18$\times$ &  7.45$\times$ & 13.36$\times$ & 24.25$\times$ & 42.04$\times$ \\
8  & 1.50 & 0.38$\times$ & 0.77$\times$ & 1.90$\times$ & 3.75$\times$ &  8.88$\times$ & 16.82$\times$ & 29.64$\times$ & 76.14$\times$ \\
16 & 3.08 & 0.45$\times$ & 0.89$\times$ & 2.21$\times$ & 4.35$\times$ & 10.54$\times$ & 20.09$\times$ & 37.32$\times$ & 76.77$\times$ \\
\bottomrule
\end{tabular}
\end{table}

%% file: figures/table_double_sparsty.tex
\begin{table}[t]
\centering
\caption{\textbf{Sparse decode is net-positive with indexer cost included. We use DoubleSparsity as an example here.}
Speedup over FlashInfer on H100 (FP16, 128K, page size 16) using Double~Sparsity~\cite{yang2024posttraining} ($8\times16$-bit channels, untuned) for index selection, under MHA ($H_q{=}H_{kv}{=}32$) and GQA ($H_q{=}32, H_{kv}{=}8$).
MHA: break-even at $2\times$ sparsity with  $4.17\times$ speedup at $100\times$ sparsity.
GQA: break-even at $10\times$, with $2.81\times$ speed up at $100\times$.
A lighter indexer (HashAttention, PQCache, low-precision Double Sparsity) is expected to widen the margin. The upper limit can be seen from Table~\ref{tab:backend_speedup}}
\label{tab:double-sparsity-speedup}
\setlength{\tabcolsep}{5pt}
\begin{tabular}{c r | r r r r r r}
\toprule
$B$ & FlashInfer (ms) & 2x & 5x & 10x & 20x & 50x & 100x \\
\midrule
\multicolumn{8}{l}{\textit{GQA\quad($H_q{=}32,\ H_k{=}8$)}} \\
\midrule
1  & 0.19 & 0.28$\times$ & 0.56$\times$ & 0.83$\times$ & 1.12$\times$ & 1.46$\times$ & 1.65$\times$ \\
4  & 0.72 & 0.32$\times$ & 0.67$\times$ & 1.06$\times$ & 1.52$\times$ & 2.11$\times$ & 2.45$\times$ \\
8  & 1.50 & 0.36$\times$ & 0.75$\times$ & 1.18$\times$ & 1.68$\times$ & 2.30$\times$ & 2.66$\times$ \\
16 & 3.08 & 0.41$\times$ & 0.85$\times$ & 1.31$\times$ & 1.82$\times$ & 2.46$\times$ & 2.81$\times$ \\
\midrule
\multicolumn{8}{l}{\textit{MHA\quad($H_q{=}H_k{=}32$)}} \\
\midrule
1  &  0.70 & 0.91$\times$ & 1.62$\times$ & 2.18$\times$ & 2.68$\times$ & 3.14$\times$ & 3.37$\times$ \\
4  &  2.79 & 1.02$\times$ & 1.86$\times$ & 2.56$\times$ & 3.17$\times$ & 3.74$\times$ & 4.00$\times$ \\
8  &  5.61 & 1.12$\times$ & 2.00$\times$ & 2.71$\times$ & 3.32$\times$ & 3.86$\times$ & 4.11$\times$ \\
16 & 11.18 & 1.22$\times$ & 2.13$\times$ & 2.84$\times$ & 3.43$\times$ & 3.94$\times$ & 4.17$\times$ \\
\bottomrule
\end{tabular}
\end{table}

%% file: sections_arxiv/kernels.tex
A common argument against sparsity concerns its poor alignment with modern hardware, which has led to considerable debate over whether block sparsity is a necessary condition for practical efficiency gains. DeepSeek-V3 Attention \cite{deepseek2025v32} demonstrates that token-level fine-grained sparsity can yield meaningful efficiency improvements at both prefill and decode time: a result with implications not only for inference but also for training-time reductions on long-context workloads. We take this a step further in the context of decode-time sparsity. We show that it is possible to improve upon the state-of-the-art FlashInfer\footnote{\url{flashinfer.ai}} decoding kernels using an even finer-grained sparsity pattern: per-query, per-query-head token-level sparsity. We find this effective even in challenging settings such as Grouped Query Attention (GQA)~\cite{gqa} where number of query heads can be much larger than key-value heads.
\\
\noindent
We benchmark sparse decode backend kernel against a full dense decode baseline on an NVIDIA H100 80GB HBM3 GPU using fp16 precision, $H_{kv}=8$, $H_q=32$, head dimension $D=128$, page size 16, and NHD layout. Table~\ref{tab:backend_speedup} shows the performance of sparse backend which computes the weighted attention given sparse index and associated weights. It shows that we can leverage even this irregular sparsity. At $50$--$100\times$ sparsity, our backend delivers $5.5$--$20\times$ speedup over FlashInfer across batch sizes; at extreme $500\times$ sparsity, speedup reaches $76\times$ at large batch (Table~\ref{tab:backend_speedup}).
\\
\noindent
To include some form of indexing mechanism, we simulate an 8-channel ($16$-bit precision) Double Sparsity~\cite{yang2024posttraining} indexer and report results in Table~\ref{tab:double-sparsity-speedup}. Double Sparsity achieves up to $4.17\times$ speedup in MHA and $2.81\times$ in GQA at $100\times$ sparsity. MHA crosses break-even at $2\times$ sparsity; GQA at $10$--$20\times$. A lighter indexer (HashAttention~\cite{desai2025hashattention}, PQCache~\cite{zhang2025pqcache}, or low-precision Double Sparsity~\cite{yang2024posttraining}) is expected to widen these margins.

%% file: sections_arxiv/conclusion.tex
The AI workload landscape is rapidly shifting toward long-context understanding and long-form generation, with tasks such as repository-scale code comprehension, long-document question answering, and agentic systems becoming increasingly common. We argue that standard attention mechanisms were not designed for such extreme context lengths. In particular, attention faces a fundamental embedding bottleneck: a relatively small hidden dimension $d \ll N$ forces information from an $N$-dimensional context to collapse into a much lower-dimensional representation. Motivated by this limitation, we envision a future in which long-context LLM inference becomes entirely sparse in the context dimension.
To support this vision, we demonstrate the surprising robustness of new-generation large models to extreme sparsity, even though these models were not explicitly trained for sparse context processing. We further show that multiple forms of sparsity can already be exploited effectively on current hardware, while even greater gains may be unlocked through future hardware designs that explicitly acknowledge the inherently sparse nature of context processing.
Overall, we believe the community should treat sparsity as a central principle when designing the next generation of model architectures, inference and training systems, and the hardware platforms on which they operate.

%% file: sections_arxiv/theory_app.tex
\newpage
\section*{Appendix}
\section{Proofs}

\subsection{Proof of Theorem~\ref{thm:dense-attention-hidden-collapse}}

\begin{proof}
Consider the linear map
\(T:\mathbb{R}^N\to \mathbb{R}^d\) defined by
\(T(a):=V^\top a\).
Since \(V^\top\in\mathbb{R}^{d\times N}\), we have
\(\operatorname{rank}(T)\le d\). Let
\(
\mathcal{N}(T):=\{z\in\mathbb{R}^N:V^\top z=0\}
\)
be the null space of \(T\). By rank-nullity,
\[
\dim \mathcal{N}(T)
=
N-\operatorname{rank}(T)
\ge
N-d.
\]

We want a nonzero direction in the null space that also preserves the simplex sum constraint. Let
\(
\mathbf{1}:=(1,\dots,1)^\top\in\mathbb{R}^N
\)
and consider the zero-sum hyperplane
\(
H:=\{z\in\mathbb{R}^N:\mathbf{1}^\top z=0\}.
\)
This hyperplane has dimension \(N-1\). By the standard dimension formula for subspaces,
\[
\dim(\mathcal{N}(T)\cap H)
\ge
\dim\mathcal{N}(T)+\dim H-N.
\]
Since \(\dim\mathcal{N}(T)\ge N-d\) and \(\dim H=N-1\), we get
\[
\dim(\mathcal{N}(T)\cap H)
\ge
(N-d)+(N-1)-N
=
N-d-1>0,
\]
where the last inequality uses \(d<N-1\). Hence there exists a nonzero
\(z\in \mathcal{N}(T)\cap H\) such that
$V^\top z=0$,
and
$\mathbf{1}^\top z=0$.
%
Since \(z\neq 0\), we may rescale it as
\(
z\leftarrow \frac{z}{\|z\|_\infty},
\)
so that \(\|z\|_\infty=1\). This rescaling preserves
\(V^\top z=0\) and \(\mathbf{1}^\top z=0\).
Now, let
\(
a_0:=\frac{1}{N}\mathbf{1}
\)
be the uniform attention distribution, and fix any \(\beta\in(0,1)\) and define
\[
a:=a_0+\frac{\beta}{N}z,
\qquad
a':=a_0-\frac{\beta}{N}z.
\]
Since \(\mathbf{1}^\top z=0\), both \(a\) and \(a'\) sum to one. Moreover, since \(\|z\|_\infty=1\), each coordinate satisfies
\[
\frac{1-\beta}{N}
\le
a_i,a'_i
\le
\frac{1+\beta}{N}.
\]
Because \(\beta\in(0,1)\), all coordinates are strictly positive. Thus \(a\) and \(a'\) are valid attention distributions, and both have full support. They are distinct because \(z\neq 0\).

Finally,
\[
V^\top a
=
V^\top a_0+\frac{\beta}{N}V^\top z
=
V^\top a_0
=
V^\top a_0-\frac{\beta}{N}V^\top z
=
V^\top a'.
\]
Thus there exist two distinct full-support attention distributions \(a\) and \(a'\) such that
\[
a\neq a',
\qquad
V^\top a=V^\top a'.
\]
Therefore, the map \(a\mapsto V^\top a\) is not injective on the attention simplex when \(d<N-1\).
\end{proof}

\section{SWE-Bench Django: subgroup, failure-mode, and cost breakdown}
\label{app:swe}

This appendix reports the complete SWE-Bench Django evaluation behind the headline of Sec.~\ref{sec:emergent}: that on tasks the dense baseline can solve, sparse attention matches dense within $\sim\!2$pp, and the larger $S_0$ gap is driven by serving-stack failures rather than attention quality.
We compare three configurations on Qwen3.5-27B served via vLLM under the \texttt{mini-swe-agent} v$2.2.8$ harness (\texttt{step\_limit=250}, \texttt{cost\_limit=\$3}, $60$s per-command timeout):
\emph{dense} is full softmax ($100\%$ density);
\emph{$5\times$} is Sink($128$) $+$ Local($128$) $+$ OracleTopK with heavy fraction $0.20$ (achieved density $\sim\!22\%$; per-layer attention-output $L_2$ error $\sim\!1.3\%$ relative to full attention);
\emph{$50\times$} uses the same scaffold with heavy fraction $0.02$ ($\sim\!3.8\%$ density, $\sim\!8.8\%$ error).
All three were run on the full $114$-instance Django subset of SWE-Bench Lite. The harness graded $114 / 113 / 110$ instances, respectively; the $1$- and $4$-instance deficits trace to serving-stack exit codes documented below.

\paragraph{Outcome definitions (SWE-Bench harness).}
\texttt{resolved}: patch applied and target tests pass.
\texttt{unresolved}: patch applied, tests fail.
\texttt{empty\_patch}: the agent produced no diff (or only test-file changes, which the harness strips before applying).
\texttt{error}: a patch was produced but \texttt{git apply} rejected it (malformed hunk or wrong line numbers).

\paragraph{Exit-status definitions (mini-swe-agent terminal state).}
The harness's outcome bucket conflates several mechanisms. Inside \texttt{empty\_patch}, the agent's terminal state isolates root cause:
\texttt{Submitted} -- agent reached the submit step and emitted an empty (or test-only) diff;
\texttt{LimitsExceeded} -- hit \texttt{step\_limit=250} without submitting (model-side: agent could not converge);
\texttt{InternalServerError} -- the vLLM server crashed mid-conversation, attributable to sparse-attention-hub instability under sustained decode (server-side, sparse-specific);
\texttt{Timeout} -- LiteLLM $1800$s connection timeout against a stalled vLLM server (same root cause as \texttt{InternalServerError});
\texttt{CalledProcessError} -- the per-instance \texttt{docker run} returned $125/127$ before the agent could place its first model call (pure infra; no sparsity dependence).

\begin{table}[ht]
\centering
\caption{\textbf{Resolution rate by subgroup (resolved/total).}
$S_0$ is the unconditional union; $S_3$ is the head-to-head subset where all three configurations emitted a non-empty patch that the harness scored. The $\sim\!10$pp $S_0$ dense$\!\to\!$sparse gap collapses to $\sim\!2$pp on $S_3$ ($5\times$ in fact slightly edges dense), placing the gap on \emph{whether} the agent emits a parseable patch rather than on \emph{patch quality}.}
\label{tab:swe-subgroups}
\setlength{\tabcolsep}{6pt}
\begin{tabular}{l c c c c}
\toprule
Subgroup & $n$ & dense & $5\times$ ($\sim\!22\%$) & $50\times$ ($\sim\!3.8\%$) \\
\midrule
$S_0$ \;union                                  & 114 & $70/114 = 61.4\%$ & $60/113 = 53.1\%$ & $57/110 = 51.8\%$ \\
$S_1$ \;3-way intersection                     & 109 & $66/109 = 60.6\%$ & $57/109 = 52.3\%$ & $56/109 = 51.4\%$ \\
$S_2 = S_1 \cap$ no harness errors             & 106 & $64/106 = 60.4\%$ & $55/106 = 51.9\%$ & $56/106 = 52.8\%$ \\
$S_3 = S_2 \cap$ all emitted non-empty patches &  58 & $45/58  = 77.6\%$ & $46/58  = 79.3\%$ & $44/58  = 75.9\%$ \\
\bottomrule
\end{tabular}
\end{table}

\paragraph{Where does the $S_0$ gap come from?}
The \texttt{empty\_patch} bucket determines the $S_0$ gap, but it is a misleadingly uniform label: the $13 / 34 / 26$ \texttt{empty\_patch} cases come from very different root causes across configurations (Tab.~\ref{tab:swe-exit-status}). For dense, the dominant root cause is \texttt{CalledProcessError}: the docker container for the instance failed to launch and the agent never placed a model call, so attention played no role in the failure. For sparse, the dominant root cause is \texttt{InternalServerError} (plus \texttt{Timeout}): the vLLM server crashed or stalled mid-trajectory under sustained sparse-attention-hub decode, again with no bearing on attention quality, just on serving-stack stability. The only failure mode plausibly attributable to attention is \texttt{LimitsExceeded} -- the agent ran out of $250$ turns without converging -- and it scales modestly with sparsity ($4 / 11 / 11$), which is the residual attention-quality cost of compressing context. On the \texttt{resolved} and \texttt{unresolved} buckets every trajectory is \texttt{Submitted}: once the agent reaches submit, sparsity does not change \emph{whether} the patch passes, only \emph{which} tasks the agent reaches submit on. The three \texttt{error} entries at $50\times$ are \texttt{git apply} rejections of malformed hunks; we treat them as the same family as \texttt{empty\_patch} -- patches the harness never scored.

\begin{table}[ht]
\centering
\caption{\textbf{Per-(outcome $\times$ config) exit-status counts.}
$n$ is the size of that outcome bucket for that configuration.
\emph{Bold} entries are root causes attributable to infrastructure (docker daemon, vLLM server crash, or connection timeout) and are eligible for retry; plain entries are model-side failures (\texttt{Submitted} but tests fail, or \texttt{LimitsExceeded}).}
\label{tab:swe-exit-status}
\setlength{\tabcolsep}{5pt}
\small
\begin{tabular}{l l r l}
\toprule
config & outcome & $n$ & exit-status mix \\
\midrule
dense   & resolved    & 70 & Submitted:$70$ \\
dense   & unresolved  & 31 & Submitted:$31$ \\
dense   & empty\_patch& 13 & \textbf{CalledProcessError:$9$}, LimitsExceeded:$4$ \\
\midrule
$5\times$  & resolved    & 60 & Submitted:$60$ \\
$5\times$  & unresolved  & 19 & Submitted:$19$ \\
$5\times$  & empty\_patch& 34 & \textbf{InternalServerError:$17$}, LimitsExceeded:$11$, Submitted:$4$, \textbf{Timeout:$2$} \\
\midrule
$50\times$ & resolved    & 57 & Submitted:$57$ \\
$50\times$ & unresolved  & 24 & Submitted:$24$ \\
$50\times$ & empty\_patch& 26 & LimitsExceeded:$11$, \textbf{InternalServerError:$11$}, \textbf{Timeout:$4$} \\
$50\times$ & error       &  3 & Submitted:$3$ (\texttt{git apply} rejected) \\
\bottomrule
\end{tabular}
\end{table}

\paragraph{Counterfactual: what would a clean rerun show?}
If we retry the infrastructure-attributable failures and assume each retry resolves at the configuration's empirical $S_3$ rate ($77.6$/$79.3$/$75.9\%$), the projected post-retry rates become dense $61.4\!\to\!66.7\%$ ($\sim\!6$ of $9$ docker failures recovered), $5\times$ $53.1\!\to\!61.1\%$ ($\sim\!9$ of $19$ vLLM-side failures recovered), and $50\times$ $51.8\!\to\!59.1\%$ ($\sim\!8$ of $15$). The sparse-to-dense gap narrows by $\sim\!4$pp but a $\sim\!6$pp residual remains: this residual is genuine server-side instability of the sparse-attention-hub backend, not attention-quality loss, and is the right target for follow-up engineering rather than for re-evaluating the sparsity claim.

\paragraph{Per-task compute cost.}
Tab.~\ref{tab:swe-cost} reports mean turns and total tokens by outcome.
On \emph{resolved} tasks, sparse runs use $\sim\!15\%$ fewer turns ($67\!\to\!57\!\to\!55$) and $\sim\!15$--$19\%$ fewer total tokens ($1.34$M $\!\to\!1.14$M $\!\to\!1.08$M); on \emph{unresolved} tasks the reduction is larger ($92\!\to\!67/70$ turns; $2.66$M $\!\to\!1.56/1.62$M tokens), suggesting the sparse agent commits to or abandons a fix sooner rather than churning. The two effects compound with the per-decode kernel speedup (Tab.~\ref{tab:backend_speedup}): wall-clock savings per productive task are the product of fewer turns, fewer tokens per turn, and faster per-call attention.

\begin{table}[ht]
\centering
\caption{\textbf{Per-task compute cost by outcome.}
Mean agent turns and mean total tokens per task.
$n$ columns report instance counts as (dense / $5\times$ / $50\times$).
\texttt{empty\_patch} and \texttt{error} costs are inflated by stalled trajectories that exhaust the step budget before the harness gives up; they are not comparable to the productive outcomes.}
\label{tab:swe-cost}
\setlength{\tabcolsep}{6pt}
\begin{tabular}{l c | r r r | r r r}
\toprule
& & \multicolumn{3}{c|}{turns (mean)} & \multicolumn{3}{c}{total tokens (mean)} \\
outcome & $n$ (d/$5\times$/$50\times$) & dense & $5\times$ & $50\times$ & dense & $5\times$ & $50\times$ \\
\midrule
resolved        & 70 / 60 / 57 & $67$  & $57$  & $55$  & $1.34$M  & $1.14$M & $1.08$M \\
unresolved      & 31 / 19 / 24 & $92$  & $67$  & $70$  & $2.66$M  & $1.56$M & $1.62$M \\
empty\_patch    & 13 / 34 / 26 & $250$ & $138$ & $192$ & $14.79$M & $6.28$M & $9.44$M \\
error           &  0 /  0 /  3 & --    & --    & $55$  & --       & --      & $1.14$M \\
\bottomrule
\end{tabular}
\end{table}

\begin{figure}[ht]
    \centering
    \includegraphics[width=\linewidth]{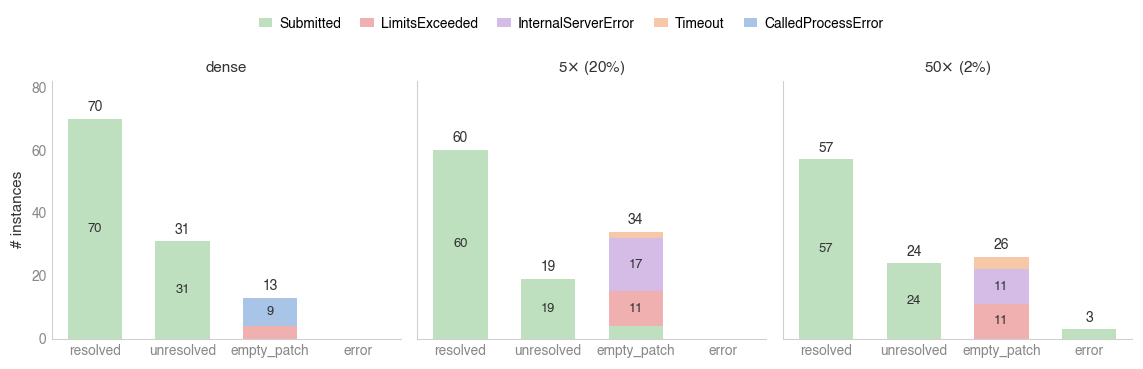}
    \caption{\textbf{Empty-patch root cause changes with attention configuration.}
    Outcome counts split by terminal exit status (visual companion to Tab.~\ref{tab:swe-exit-status}; numerical counts there).
    The takeaway is the colour composition of the \texttt{empty\_patch} bar: dense's \texttt{empty\_patch} is almost entirely docker-launch failures (blue, \texttt{CalledProcessError}); sparse's is almost entirely vLLM crashes and timeouts (purple/orange, \texttt{InternalServerError}/\texttt{Timeout}). The only stratum that grows with sparsity is \texttt{LimitsExceeded} (pink, $4 \!\to\!11 \!\to\!11$), the model-side residual.}
    \label{fig:swe-outcomes-exit}
\end{figure}

\begin{figure}[ht]
    \centering
    \includegraphics[width=0.9\linewidth]{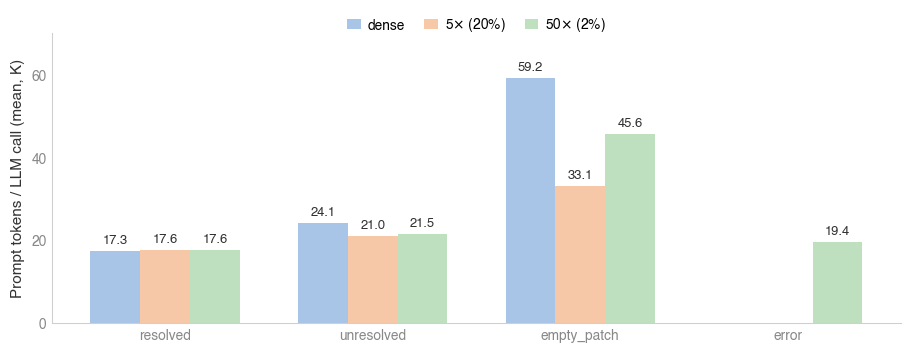}
    \caption{\textbf{Per-LLM-call prompt size is unaffected by sparsity on productive outcomes.}
    Mean prompt tokens per LLM call, by outcome and configuration.
    On \texttt{resolved} the three configurations are within $0.3$K tokens/call ($17.3 / 17.6 / 17.6$K); on \texttt{unresolved} they sit within $3$K ($24.1 / 21.0 / 21.5$K). Sparsity does not change how much context the agent maintains per call -- it changes what attention does \emph{with} that context. The large \texttt{empty\_patch} bars ($59.2 / 33.1 / 45.6$K) come from the stalled-trajectory tail (Fig.~\ref{fig:swe-outcomes-exit}) where the agent burns through the $250$-turn cap on a steadily growing prompt; the elevated tokens reflect the failure mode, not the productive cost.}
    \label{fig:swe-tokens-per-turn}
\end{figure}